\documentclass{article}

\usepackage{microtype}
\usepackage{graphicx}
\usepackage{subfigure}
\usepackage{booktabs} 
\usepackage{todonotes}
\definecolor{lilac}{rgb}{0.8, 0.6, 0.8}

\usepackage{hyperref}

\usepackage[accepted]{icml2024}

\usepackage{amsmath}
\usepackage{amssymb}
\usepackage{mathtools}
\usepackage{amsthm}

\usepackage[capitalize,noabbrev]{cleveref}

\theoremstyle{plain}
\newtheorem{theorem}{Theorem}[section]
\newtheorem{proposition}[theorem]{Proposition}

\theoremstyle{definition}

\theoremstyle{remark}

\icmltitlerunning{Relaxed Equivariant Graph Neural Networks}

\begin{document}

\twocolumn[
\icmltitle{Relaxed Equivariant Graph Neural Networks}



\icmlsetsymbol{equal}{*}

\begin{icmlauthorlist}
\icmlauthor{Elyssa Hofgard}{1}
\icmlauthor{Rui Wang}{1}
\icmlauthor{Robin Walters}{2}
\icmlauthor{Tess Smidt}{1}
\end{icmlauthorlist}

\icmlaffiliation{1}{Department of Electrical Engineering and Computer Science, MIT, USA}
\icmlaffiliation{2}{College of Computing Sciences, Northeastern University, USA}

\icmlcorrespondingauthor{Elyssa Hofgard}{ehofgard@mit.edu}

\icmlkeywords{Equivariant machine learning, symmetry breaking, relaxed convolution}
\vskip 0.2in
]



\printAffiliationsAndNotice{} 

\begin{abstract}
3D Euclidean symmetry equivariant neural networks have demonstrated notable success in modeling complex physical systems. We introduce a framework for relaxed $E(3)$ graph equivariant neural networks that can learn and represent symmetry breaking within continuous groups. Building on the existing \texttt{e3nn} framework, we propose the use of relaxed weights to allow for controlled symmetry breaking. We show empirically that these relaxed weights learn the correct amount of symmetry breaking.
\end{abstract}

\section{Introduction}
3D Euclidean symmetry equivariant neural networks or $E(3)$NNs \citep{cohen2016steerable, thomas2018tensor, Kondor2018-nbody, weiler20183d} have had considerable success in modeling complex physical data \textemdash from learning ab initio molecular dynamics \citep{Batzner2022-sr} to predicting quantum mechanically accurate properties of molecules and crystals \citep{Rackers2023-sb, fang2024phonon}. Through integrating physical symmetries into the model architecture, equivariant neural networks can achieve superior generalization and data efficiency \citep{liao2023equiformerv2, Batzner2022-sr, Frey2023-gg, Rackers2023-sb, Owen2023-xw}.

However, in many physical settings, there may be symmetry breaking or approximate symmetries. For instance, in crystal phase transitions, spontaneous symmetry breaking marks the shift from a high-symmetry state (like a liquid) to a low-symmetry, ordered crystalline structure, fundamentally altering the material's properties. Symmetry breaking can also occur when external forces act on the system. For example, in fluid dynamics, external forces or temperature gradients can break the Euclidean symmetry of a uniform fluid layer \cite{tagawa2023symmetry}. It is therefore desirable to develop models capable of parametrizing sources of symmetry breaking for applications across diverse physical systems. In this work, we refer to these sources of asymmetry as symmetry-breaking factors. It is noteworthy that these factors are analogous to order parameters in Landau theory, which describe the system parameters leading to phase transitions \cite{landau1936orderparam}. Previous works \citep{wang2022approximately, wang2024symmbreak} have developed relaxed group convolutions for discrete groups to learn and parametrize symmetry breaking. We develop a relaxed equivariant graph convolution neural network architecture, generalizing to the continuous group $E(3)$ using operations present in \texttt{e3nn} \cite{e3nnsoftware, geiger2022e3nn}. 


\section{Background}
We begin with a brief overview of the group representation theory and notation relevant to convolutions in $E(3)$ equivariant neural networks. Euclidean symmmetry in 3D includes 3D translations, rotations, and inversion. Translation equivariance is typically achieved by using convolutions. Convolutional filters are thus constrained to be equivariant to $O(3) = SO(3) \times \mathbb{Z}_2$, the group of 3D rotations and inversion. In the following sections, we thus discuss $O(3)$.

\textbf{Group representations.} Group representations provide a concrete way to handle abstract symmetry groups. A representation of a group $G$ on vector space $V$ of dimension $d$ is a function $D : G \to \text{GL}(V)$. $\text{GL}(V)$ is the group of invertible $d \times d$ matrices. A representation $D$ has the following properties $\forall g,h \in G$ and where $e \in G$ is the identity element: (i) $D(gh) = D(g)D(h)$ and (ii) $D(g)D(g^{-1}) = D(g^{-1})D(g) = D(e) =  \mathbf{I}_d$.

\textbf{Irreducible representations.} An irreducible representation (irrep) is a representation that does not contain a smaller representation\textemdash there is no nontrivial projector $P \in \mathbb{R}^{q \times d}$ such that $g \to PD(g)P^T$ is also a representation \citep{Dresselhaus2008}. In $E(3)$NNs, data is typed by how it transforms under given irreps. 

The irreps of $O(3)$ are the product of the irreps of $SO(3)$ and $\mathbb{Z}_2$. For $SO(3)$, the irreps are indexed by a positive integer angular frequency $l = 0,1,2,\dots$. For $\mathbb{Z}_2$, there are two irreps: the even irrep indexed by parity $p=1$ and the odd irrep indexed by parity $p=-1$. The irreps of $O(3)$ thus carry both $l$ and $p$ indices and can be associated with matrices $D^{l,p}$. Each irrep has an associated matrix $D^{l,p}$ which operates on the vector space $V^l$ of dimension $2l+1$. In this work, we will use tuples $(l, p)$ and symbols $l_p$ to indicate specific irreps. The former is useful when dealing with tensor indices and the later is easier when explicitly listing direct sums or concatentations of irreps.

\textbf{Spherical harmonics.} Convolutions that are equivariant to $O(3)$ are built from spherical harmonics. The spherical harmonics are a family of functions $Y^l$ from the unit sphere to the vector space of the irrep $D^{l,p}$, $Y^l: S^2 \rightarrow V^{l}$, and form a basis for all equivariant polynomials on the sphere \cite{geiger2022e3nn}. The $2l + 1$ spherical harmonics for a given $l$ are denoted by $Y^l_m$ where $m = -l, \dots, 0, \dots l$. The spherical harmonics have specific parity, with even $l$ spherical harmonics being even under inversion ($p=1$) and odd $l$ spherical harmonics being odd under inversion ($p=-1)$.
The spherical harmonics are equivariant under $O(3)$ and thus transform via $D^{l, p}$:
\begin{equation} \label{eq:sphharm}
    Y^l_m(R(g) \vec{r}) = D^{l, p}(g)_{mk} Y^l_k(\vec{r})
\end{equation}
where $\vec{r}$ is a 3D vector and $R$ is the representation of $O(3)$ on 3D vectors (a rotation matrix).

\textbf{$\mathbf{E(3)}$ equivariant graph convolutions.} The primary distinction between traditional convolutions and equivariant convolutions lies in the operations between filters and inputs. Equivariant convolutions must ensure that any operation respects underlying symmetries and faithfully treats higher-order geometric objects (beyond scalars). $E(3)$NNs thus use tensor product decompositions.

Tensor product decompositions multiply two direct sums of irreps (input and filter) and provides the change of basis back into a new direct sum of irreps (output) in way that preserve the multiplication rules of the group. The product of two irreps can produce multiple different irreps, referred to as different ``paths'' of interaction that are independently equivariant. At a high level, these tensor product decompositions are encoded by a three index tensor built from Clebsch-Gordan coefficients. For two inputs $x$ and $y$, $x \otimes y = \sum_{\alpha \beta} C_{\alpha \beta \gamma} x_{\alpha} y_{\beta} = z$ such that we are guaranteed the property that $D^{Z}(g)z = D^X(g)x \otimes D^Y(g) y$. See \citet{geiger2022e3nn} for more detail.

While in \citet{thomas2018tensor, Kondor2018-nbody, weiler20183d}, the convolutional filter is expressed in terms of a radial basis multiplying the spherical harmonics, in more recent networks such as those implemented in \texttt{e3nn}, radial functions are used to provide weights to the tensor product decomposition, as each independent ``path'' can be weighted by a scalar and still preserve equivariance \cite{geiger2022e3nn}. We adopt the following notation for equivariant convolutions.  
\begin{align} 
\tilde{Y}(\hat{r}_{ab}) &= \bigoplus_{0=l}^{l_{\text{max}}} Y_l(\hat{r}_{ab})
\\
    f_a^{\prime} &= \frac{1}{\sqrt{N}} \sum_{b \in \delta(a)} f_b \otimes_{W(||r_{ab}||)} \tilde{Y}(\hat{r}_{ab}) \label{eq:conv}
\end{align}

$f_b, f_a^{\prime}$ are the input and output node features. $N$ is the average degree of the nodes, $\hat{r}_{ab}$ is the relative unit vector from position $a$ to $b$, $\delta(a)$ is the set of neighbor nodes of node $a$, and $\tilde{Y}$ is the entire set of spherical harmonics used in the filter from $l=0$ to some $l_{\text{max}}$. The spherical harmonic projection $\tilde{Y}(\hat{r}_{ab})$ refers to evaluating the spherical harmonics at $\hat{r}_{ab}$ at each $Y^l, 0 \leq l \leq l_{\text{max}}$ and concatenating the results.
The learnable radial function that parameterizes the tensor product $W(||r_{ab}||)$ is a multi-layer perceptron (MLP) acting on a set of $b$ radial basis functions $B: \mathbb{R} \rightarrow \mathbb{R}^b$ evaluated at the magnitude $||r_{ab}||$.
$W$ has the form:
\begin{align}
W := \text{MLP}(B(||r_{ab}||)_{((c, l, p)_i, (c, l, p)_f, (c, l, p)_o)}
\end{align}
where $c$ indicate channel indices, as there can be multiple copies of a given irrep in the input or output. $(l, p)$ with indices $i,f,o$ specify the irreps in the input, filter, and output respectively. Tallowed combinations of $l_i, l_f$, and $l_o$ are determined by the group multiplication rules (Clebsch-Gordan Coefficients). Common choices of radial basis functions include evenly space, partially overlapping Gaussians or Bessel functions. In our work, we expand upon this $E(3)$ equivariant graph convolution.

\section{Related Work}

\citet{wang2022approximately} proposed relaxed group convolution for discrete groups. \citet{wang2024symmbreak} extended relaxed group convolutions to show that they can allow the model to maintain the highest level of equivariance consistent with the data and discover symmetry-breaking factors. We propose a formulation of relaxed convolutions for continuous rather than discrete groups. \citet{smidt2021finding} demonstrated that the gradients to ENNs can be used to learn the appropriate symmetry breaking factors through providing an additional trainable input. Our approach is complimentary in that we allow the weights to break symmetry rather than training an additional input to the network. \citet{mcneela2023almost} proposes a Lie algebra convolution to relax the equivariance bias but does not show how to recover symmetry breaking factors.


\section{Methodology}
\subsection{Definition of relaxed $E(3)$NN}
Here, we present the definition of our \textbf{relaxed} $E(3)$ equivariant graph neural network or relaxed $E(3)$NN. The relaxed weights $\tilde{\theta}$ are the direct sum or concatenation of scalar and non-scalar irreps (which if non-zero break $O(3)$ symmetry). For example, if we wanted to include all irreps up to some angular frequency cutoff $l_{\text{relaxed}}$, $\tilde{\theta}$ would be
\begin{align}
\tilde{\theta} = \bigoplus_{
\substack{0\le l \le l_{\text{relaxed}} \\ p \in \{1, -1\}}} \theta^{l, p}
\end{align}
where $l$ and $p$ indicate the angular frequency and parity of the irrep of $O(3)$, respectively. As irreps are $2l+1$ dimensional, the overall dimensionality of $\tilde{\theta}$ depend both on the number and types of non-scalar irreps it is built from.
We interact our learned relaxed weights with the spherical harmonic projection of relative distance vectors via a tensor product $\tilde{\theta} \otimes \tilde{Y}(\hat{r}_{ab})$. The full equation for relaxed convolution is then

\begin{align} \label{eq:relax_conv}
    f_a^{\prime} = \frac{1}{\sqrt{z}} \sum_{b \in \delta(a)} f_b \otimes_{W(||r_{ab}||)} (\tilde{\theta} \otimes \tilde{Y}(\hat{r}_{ab}))
\end{align}
The weights that parameterize the tensor product of the filter and input $W(||r_{ab}||)$ now depend on the direct sum of irreps formed by the tensor product of $\tilde{\theta} \otimes \tilde{Y}(\hat{r}_{ab})$. Note that the relaxed weights are shared across input and output channels, thus only introducing a few more parameters for each layer. 

We initialize the relaxed weights $\tilde{\theta}$ such that the initial model remains equivariant to ensure the model learns the minimal amount of symmetry breaking.
\begin{proposition}
The relaxed $E(3)$NN is equivariant if and only if $\tilde{\theta}$ transforms as a scalar, i.e. only the term corresponding to the scalar irrep is non-zero.
\end{proposition}
\begin{proof}
For the relaxed $E(3)$NN to be equivariant, the following must be true
\begin{align}
D^{\tilde{\theta} \otimes \tilde{Y}}(\tilde{\theta} \otimes \tilde{Y}(\hat{r}_{ab})) = (\tilde{\theta} \otimes \tilde{Y}(R(g)\hat{r}_{ab}))
\end{align}
This can only be true for any $D^{\tilde{\theta} \otimes \tilde{Y}}$ if $\tilde{\theta}$ transforms as a scalar.
\end{proof}
We thus initialize $\tilde{\theta}$ such that only the scalar irrep is non-zero. We emphasize that pseudoscalars, $\theta^{0, -1}$ change sign under inversion in $O(3)$ and are thus non-scalar irreps.
\subsection{Interpreting Relaxed Weights as Spherical Signals}
This definition of relaxed weights also provides interpretability. As spherical harmonics form the basis for functions on a sphere, irreps representing symmetry breaking can be visualized as scalar or pseudoscalar signals illustrating the amount of symmetry or asymmetry learned by the relaxed weights. The scalar signal would be given by 
\begin{align}\label{eq:scalar_sig}
    f(\vec{x}) = \sum_{
\substack{0\le l \le l_{\text{relaxed}} \\ p = \{\text{parity}(Y^0),\dots,\text{parity}(Y^{l_{\text{relaxed}}})\}}} \theta^{l,p} \cdot Y^l(\vec{x})
\end{align}
where we sum over relaxed weights with the same parity as the spherical harmonics. The pseudoscalar signal would be given by summing over relaxed weights with opposite parity to the spherical harmonics.

 
\section{Experiments}
We present experiments showing that the relaxed $E(3)$NN learns the correct symmetry breaking factors. \citet{smidt2021finding} noted that the output of an equivariant neural network will always be higher than or equal symmetry than an input. Thus, without the addition of relaxed weights or a trainable input parameter, one will not be able to deform a shape with higher symmetry into a shape with lower symmetry for example.\footnote{When referring informally to the symmetry of a sample $x$ we are referring to the stabilizer $\text{Stab}(x) = \{g \in G | g \cdot x = x\}$ or the set of symmetry operations that leave the shape fixed.} Model details are in Appendix \ref{app:exp}.
\subsection{Shape Deformations}\label{exp:shape_def}

We consider deformations of 3D shapes, expressed in points of the form $(x,y,z)$. The network takes as input the spherical harmonic projection of the vertices of the initial shape and aims to produce, as output, the spherical harmonic projection of the vertices of the transformed shape, up to some specified $l_{\max}$. We deform a cube to a cube, a cube to a rectangular prism, and a cube to an asymmetric shape, shown in Figure \ref{fig:3d_shapes}. We use a 2-layer relaxed $E(3)$NN with the network $l_{\text{max}} = 2$ and use relaxed weights up to $l_{\text{relaxed}} = 4$ with both even and odd parity in each layer ($\tilde{\theta} = 0_e \oplus 0_o \oplus 1_e \oplus 1_o \oplus 2_e \oplus 2_o \oplus 3_e \oplus 3_o \oplus 4_e \oplus 4_o$). We note that both the choice of the network $l_{\text{max}}$ and $l_{\text{relaxed}}$ is a hyperparameter. $l_{\max}=2$ is a standard choice for the network. We choose a higher $l_{\text{relaxed}}$ to illustrate different learned symmetry breaking factors. To ensure that the learned symmetry breaking includes parity, both even and odd irreps should be included.


We analyze the learned relaxed weights using conventional symmetry analysis to confirm that the model learns the correct amount of symmetry breaking, following the same process as \citet{smidt2021finding}. First consider the case of transforming the cube to the rectangular prism. In 3D, the cube has symmetry group $O_h$, and the symmetry group of the rectangular prism is $D_{4h}$. 
Intuitively, for a cube, the $x,y$ and $z$ axes are symmetrically equivalent. However, in a rectangular prism, the $x$ and $y$ axes are symmetrically equivalent, and the $z$ axis is unique. Thus, the relaxed weights should learn some scalar signal that demonstrates this relationship (i.e. $x^2 = y^2 \neq z^2$). We can calculate the learned scalar signals from the relaxed weights using Equation \ref{eq:scalar_sig}.



The learned relaxed weights for cube to rectangular prism are zero except for $\theta^{2,1}, \theta^{4,1}$, where we use the notation for irreps indexed by $(l,p)$. We can consider the scalar signal learned by $\theta^{2,1}$ using the analytic expression for the spherical harmonics given in the \texttt{e3nn} codebase. The specific non-zero relaxed weights correspond to $m=0$ and $m=2$, so we compute the scalar signal with spherical harmonics $Y_0^{2,1},Y_2^{2,1}$, where the subscript denotes the $m$ index. Computing the scalar signal $\theta_0^{2,1}Y_0^{2,1} + \theta_2^{2,1}Y_2^{2,1}$ yields a function of the form $ax^2+ay^2-2az^2$ for each layer's relaxed weights, demonstrating that the $x$ and $y$ axes are symmetrically equivalent but that the $z$ axis is unique. In fact, the learned scalar signal corresponds exactly to a basis function for irrep $E_g$ of $O_h$ that is not present in $D_{4h}$ (see Appendix \ref{app:char_table}). The same computation holds for the learned $\theta^{4,1}$ relaxed weights. Thus, we find that either $\theta^{2,1}$ or $\theta^{4,1}$ breaks symmetry from the cube to the rectangular prism. Note that one may recover degenerate symmetry breaking factors as in this case, this degeneracy can be reduced by using a smaller $l_\text{relaxed}$ or by adding a loss penalizing higher $l$ weights \cite{smidt2021finding}.
\begin{figure}[htb!]
	\centering
	\includegraphics[width=0.48\textwidth]{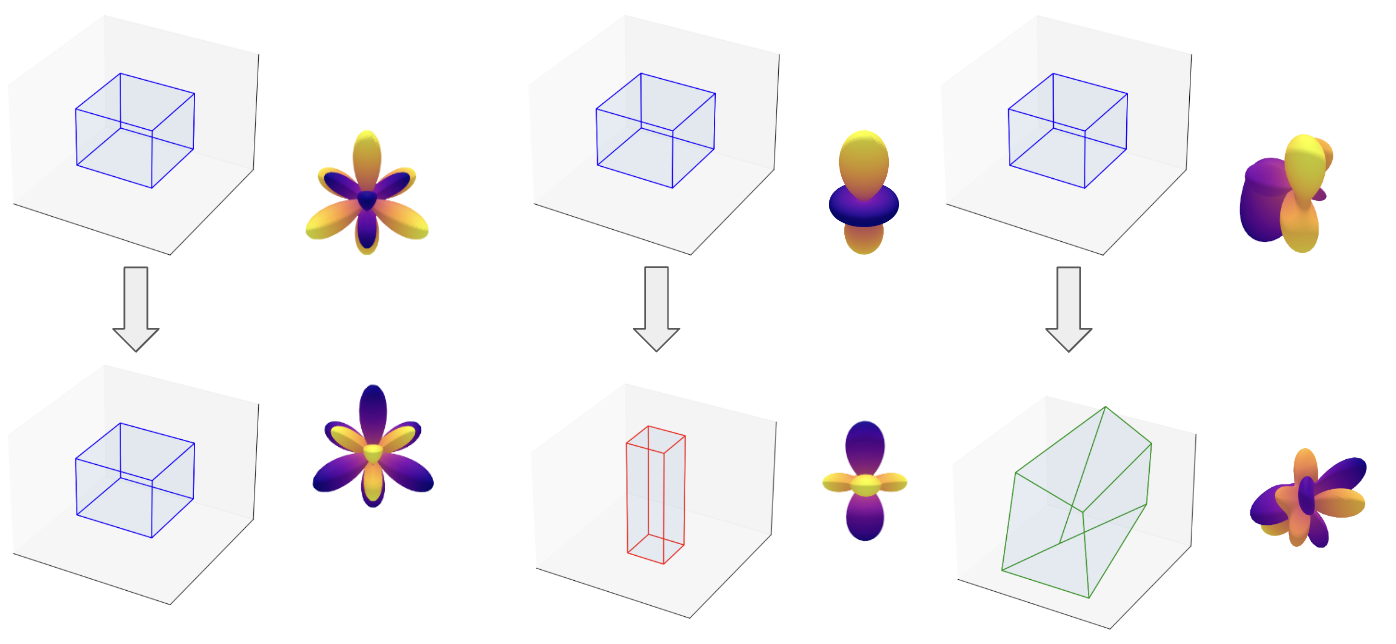}
	\caption{Visualization of tasks and corresponding spherical harmonic projections of the relaxed weights for the first (first row) and second (second row) layers. The spherical harmonic projections are plotted setting the scalar ($0_e$) term to zero for ease of viewing. A 2-layer relaxed $E(3)$NN network is trained to 1) map a cube to a cube, 2) map a cube to a rectangular prism, and 3) map a cube to a less symmetric object.}
	\label{fig:3d_shapes}
\end{figure}
\begin{figure}[htb!]
	\centering
	\includegraphics[width=0.48\textwidth]{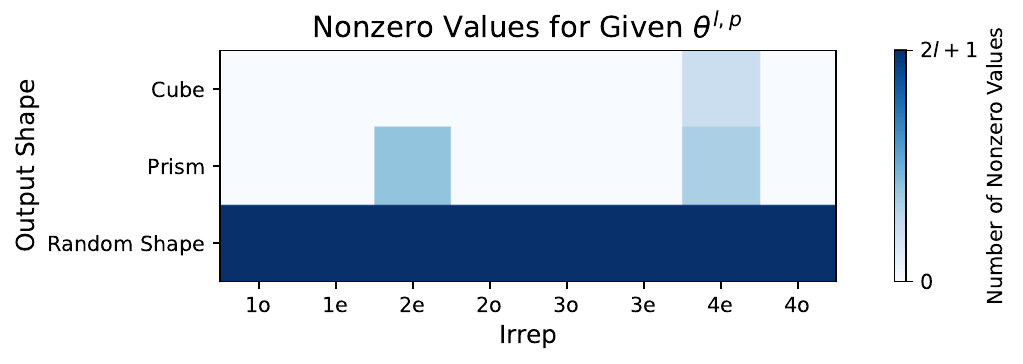}
	\caption{The number of nonzero values for each $\theta^{l,p}$ for the last layer in each network for the task of mapping the cube to a given output shape. Note that each irrep $l$ has dimension $2l+1$.}
	\label{fig:3d_sparsity}
\end{figure}

In the case of mapping the cube to itself, we do not learn any relaxed weights for $l=2$, thus preserving $O_h$ symmetry. When transforming the cube to a random shape, all relaxed weights are non-zero, illustrating the lack of symmetry as seen in Figure \ref{fig:3d_sparsity}. These experiments demonstrate the interpretability of the relaxed weights as spherical harmonic signals and that they illuminate the correct symmetry breaking factors.

\subsection{Charged Particle in an Electromagnetic Field}
To illustrate a physical application of the relaxed $E(3)$NN, we consider the trajectory of a charged particle in an electromagnetic field. The force on a charged particle in an electric field $\vec{E}$ and a magnetic field $\vec{B}$ is given by $\vec{F} = q(\vec{E} + \mathbf{v} \times \vec{B})$. $\vec{E}$ is a vector ($1_o$) and $\vec{B}$ is a pseudovector ($1_e$). Figure \ref{fig:particle_traj} presents an example of a positively charged particle trajectory with random initial position and velocity under an electric field in the $\hat{x}$ direction and a magnetic field in the $\hat{y}$ direction. 
\begin{figure}[htb!]
	\centering
	\includegraphics[width=0.3\textwidth]{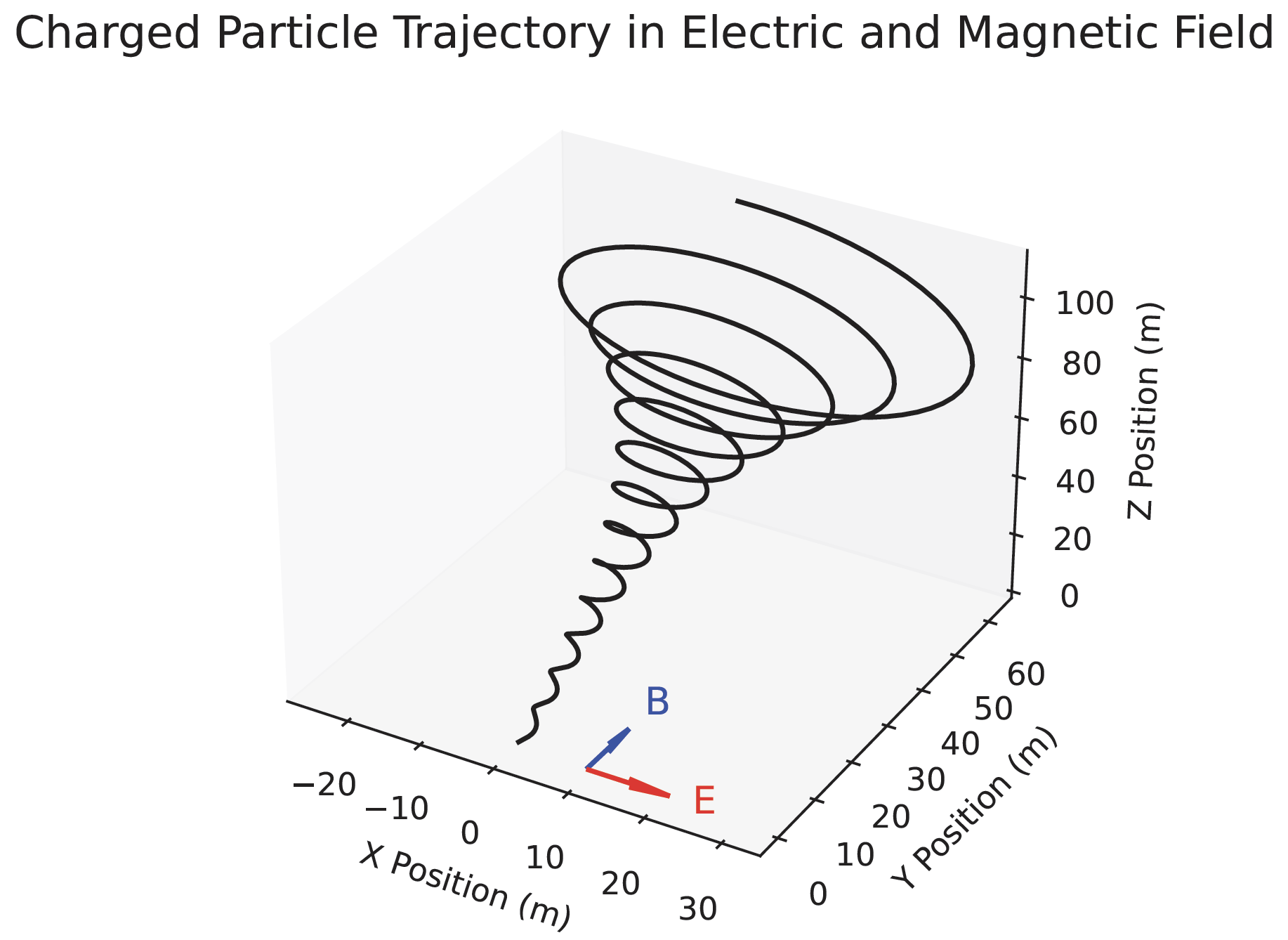}
	\caption{Sample trajectories of particles with random initial velocities and starting positions under an electric field in the $\hat{x}$ direction and a magnetic field in the $\hat{y}$ direction.}
	\label{fig:particle_traj}
\end{figure}

We thus consider if a relaxed $E(3)$NN can discover the relationship between the electric and magnetic field\textemdash learning a $1_o$ vector with the correct magnitude and direction for $\vec{E}$ and learning a $1_e$ pseudovector with the correct magnitude and direction for $\vec{B}$. For simplicity, we consider a single particle trajectory with $\vec{E} = [1,0,0]$ and $\vec{B} = [0,1,0]$ with a randomly initialized position and velocity. The network is adapted from a convolutional network for a graph with node/edge attributes (NetworkForAGraphWithAttributes included in \texttt{e3nn}).

We aim to predict the force on each particle given its velocity and thus include no coordinate information to the network. The input to the network is a single particle with a $1_o \oplus 0_e$ node attribute representing the particle velocity and the scalar charge. Note the even parity scalar charge is included to ensure the network can learn even irreps as well as odd. The desired output is the $1_o$ force acting on the particle at that timestep. We consider a 1-layer relaxed network with relaxed weights $0_e\oplus 0_o \oplus 1_o \oplus 1_e$. The network is then trained on multiple timesteps in the trajectory, predicting the force at each timestep. 

We emphasize that through showing the model multiple timesteps with different velocity vectors, we break symmetry and allow the relaxed weights to learn the form of the electric and magnetic fields. An equivariant neural network with equivalent architecture without relaxed weights cannot accomplish this task, rather the loss remains ``stuck'' in a non-symmetry breaking configuration. 
\begin{table}[]
\centering
\begin{tabular}{l|l|l}
Model & Equiv & Relaxed \\ \hline
Training MSE   & 0.07  & \textbf{1e-6}    \\ 
\end{tabular}
\caption{Training MSE for equivariant model and relaxed model when predicting electromagnetic forces for multiple timesteps in a trajectory. The relaxed model is able to overfit to the training set while the equivariant model remains in a non-symmetry breaking configuration.}
\label{table:mse}
\end{table}
The relaxed $E(3)$NN learns the correct $\vec{E}$ and $\vec{B}$ fields, with error of approximately 0.001-0.01 in the $y$ and $z$ components after normalization as seen in Figure \ref{fig:pred_e_b}. This example illustrates the physical interpretability of the relaxed weights and their potential for application to more complicated physical examples.

\begin{figure}[htb!]
	\centering
	\includegraphics[width=0.45\textwidth]{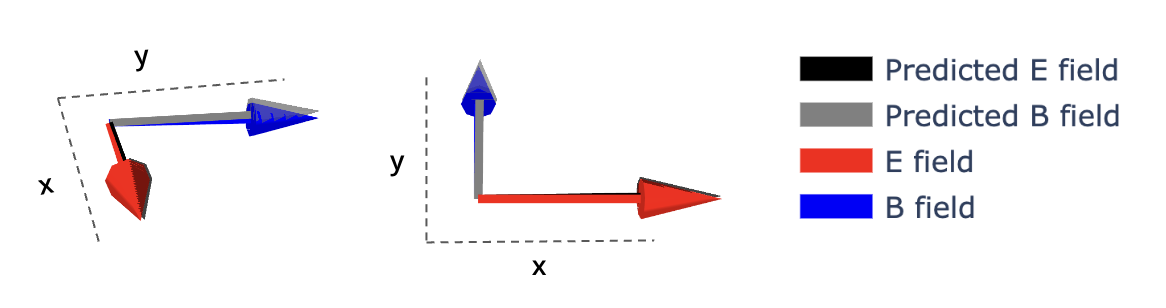}
	\caption{True and predicted fields for the relaxed network.}
	\label{fig:pred_e_b}
\end{figure}

\section{Discussion}
We introduce a formulation for a relaxed $E(3)$ equivariant graph neural network and implement this in the $\texttt{e3nn}$ framework. Using deformations of 3D shapes and analyzing the learned relaxed weights, we find that the relaxed weights learn the correct mathematical form of the symmetry breaking parameters. Given the physical example of a charged particle in an electromagnetic field, the relaxed $E(3)$NN is able to learn the vector and pseudovector form of the electric and magnetic field. We aim to further characterize the mathematical properties of relaxed $E(3)$NNs, apply them to more complicated examples in experimental physics and materials science, and study their optimization dynamics. For reproducibility, we include our code at \url{https://github.com/atomicarchitects/RelaxedE3NN}.

\section{Acknowledgements}
Elyssa Hofgard was supported by the U.S. Department of Energy, Office of Science, Office of Advanced Scientific Computing Research, Department of Energy Computational Science Graduate Fellowship under Award Number DE-SC0024386. Rui Wang and Tess Smidt were supported by DOE ICDI grant DE-SC0022215. Robin Walters was supported by NSF 2134178. This research used resources of the National Energy Research Scientific Computing Center (NERSC) under Award Number ERCAP0028753, a Department of Energy Office of Science User Facility. 

This report was prepared as an account of work sponsored by an agency of the United States Government. Neither the United States Government nor any agency thereof, nor
any of their employees, makes any warranty, express or implied, or assumes any legal liability or responsibility for the accuracy, completeness, or usefulness of any information, apparatus, product, or process disclosed, or represents that its use would not infringe privately owned rights. Reference herein to any specific commercial product, process, or service by trade name, trademark, manufacturer, or otherwise does not necessarily constitute or imply its endorsement, recommendation, or favoring by the United States Government or any agency thereof. The views and opinions of authors expressed herein do not necessarily state or reflect those of the United States Government or any agency thereof.

\pagebreak
\bibliography{references}
\bibliographystyle{icml2024}

\newpage
\appendix
\onecolumn

\section{Symmetry Analysis for Experiment \ref{exp:shape_def}}\label{app:symm_anal}
\subsection{Character Tables}\label{app:char_table}
We present an overview of character tables used for the symmetry analysis in Section \ref{exp:shape_def}. Character tables are used for finite groups (subgroups of $E(3)$) and are used extensively in solid state physics, crystallography, and chemistry. See \citet{Dresselhaus2008, zee2016group} for more information on group and representation theory in physics and materials.

\textbf{Classes and character.} Given a group $G$, two elements $g$ and $g'$ are conjugate if there exists another element $x \in G$ such that $g' = x^{-1}g x$. A class is the set of group elements that can be obtained from a given element $g \in G$ by conjugation. The elements of a given group can thus be divided into classes. Classes correspond to physically distinct kinds of symmetry \cite{Dresselhaus2008} and are thus useful for categorizing different symmetries of a material rather than enumerating all individual group elements. Examples of classes are the twofold axes of rotation of an equilateral triangle or the threefold rotations. These classes can be related to the character of a given representation. Let $V$ be a finite-dimensional vector space and consider $D : G \to \text{GL}(V)$, a representation of $G$ on $V$. The character of $D(g)$ is the trace of the matrix of the representation, $\text{Tr}(D(g))$. One can show that the character for each element in a class is the same \cite{Dresselhaus2008}. As the character for each element in a class is the same, this provides a useful way to connect group representations to sets of physically distinct symmetries.

The relationship between characters of representations and classes for a given group can be summarized in a character table. The left hand column labels the irreps and the top row labels the class. Notation used in character tables can vary, with Schoenflies symmetry notation commonly used to describe molecular symmetries and Hermann-Mauguin notation commonly used in crystallography. For the purposes of this analysis, it is not necessary to be familiar with this notation, but the interested reader may consult Chapter 3 in \citet{Dresselhaus2008}. The entries correspond to the characters of different irreps. On the right-hand side, the columns list basis functions corresponding to irreps \textemdash functions that can be used to generate matrices of given irreps and transform the same way as that irrep. 

\begin{table}
\centering
\resizebox{\columnwidth}{!}{%
\begin{tabular}{|c|c|c|c|c|c|c|c|c|c|c|c|c|c|}
\hline $\mathbf{O}_{\mathbf{h}}$ & $\mathrm{E}$ & $8 \mathrm{C}_3$ & $6 \mathrm{C}_2$ & $6 \mathrm{C}_4$ & $3 \mathrm{C}_2=\left(\mathrm{C}_4\right)^2$ & $\mathrm{i}$ & $6 \mathrm{~S}_4$ & $8 \mathrm{~S}_6$ & $3 \sigma_{\mathrm{h}}$ & $6 \sigma_d$ & \begin{tabular}{l} 
linear functions, \\
rotations
\end{tabular} & \begin{tabular}{l} 
quadratic \\
functions
\end{tabular} & \begin{tabular}{l} 
cubic \\
functions
\end{tabular} \\
\hline$A_{1 g}$ & +1 & +1 & +1 & +1 & +1 & +1 & +1 & +1 & +1 & +1 & - & $x^2+y^2+z^2$ & - \\
\hline$A_{2 g}$ & +1 & +1 & -1 & -1 & +1 & +1 & -1 & +1 & +1 & -1 & - & - & - \\
\hline $\mathrm{E}_{\mathrm{g}}$ & +2 & -1 & 0 & 0 & +2 & +2 & 0 & -1 & +2 & 0 & - & $\left(\textcolor{blue}{\mathbf{2 z^2-x^2-y^2}}, x^2-y^2\right)$ & - \\
\hline $\mathrm{T}_{1 \mathrm{~g}}$ & +3 & 0 & -1 & +1 & -1 & +3 & +1 & 0 & -1 & -1 & $\left(\mathrm{R}_{\mathrm{x}}, \mathrm{R}_{\mathrm{y}}, \mathrm{R}_{\mathrm{z}}\right)$ & - & - \\
\hline $\mathrm{T}_{2 \mathrm{~g}}$ & +3 & 0 & +1 & -1 & -1 & +3 & -1 & 0 & -1 & +1 & - & $(x z, y z, x y)$ & - \\
\hline $\mathrm{A}_{1 \mathrm{u}}$ & +1 & +1 & +1 & +1 & +1 & -1 & -1 & -1 & -1 & -1 & - & - & - \\
\hline $\mathrm{A}_{2 \mathrm{u}}$ & +1 & +1 & -1 & -1 & +1 & -1 & +1 & -1 & -1 & +1 & - & - & $\mathrm{xyz}$ \\
\hline $\mathrm{E}_{\mathrm{u}}$ & +2 & -1 & 0 & 0 & +2 & -2 & 0 & +1 & -2 & 0 & - & - & - \\
\hline $\mathrm{T}_{1 \mathrm{u}}$ & +3 & 0 & -1 & +1 & -1 & -3 & -1 & 0 & +1 & +1 & $(\mathrm{x}, \mathrm{y}, \mathrm{z})$ & - & $\left(x^3, y^3, z^3\right)\left[x\left(z^2+y^2\right), y\left(z^2+x^2\right), z\left(x^2+y^2\right)\right]$ \\
\hline $\mathrm{T}_{2 \mathrm{u}}$ & +3 & 0 & +1 & -1 & -1 & -3 & +1 & 0 & +1 & -1 & - & - & {$\left[x\left(z^2-y^2\right), y\left(z^2-x^2\right), z\left(x^2-y^2\right)\right]$} \\
\hline
\end{tabular}%
}
\caption{Character table for point group $O_h$ \citep{charactertableOh}.}
\label{table:Oh}
\end{table}

\begin{table}
\centering
\resizebox{\columnwidth}{!}{%
\begin{tabular}{|c|c|c|c|c|c|c|c|c|c|c|c|c|c|}
\hline $\mathbf{D}_{4 h}$ & $\mathrm{E}$ & $2 \mathrm{C}_4(\mathrm{z})$ & $\mathrm{C}_2$ & $2 \mathrm{C}_2^{\prime}$ & $2 \mathrm{C}_2$ & $ \mathrm{i}$ & $2 \mathrm{~S}_4$ & $\sigma_{\mathrm{h}}$ & $2 \sigma_{\mathrm{v}}$ & $ 2 \sigma_d$ & \begin{tabular}{l} 
linear functions, \\
rotations
\end{tabular} & \begin{tabular}{l} 
quadratic \\
functions
\end{tabular} & \begin{tabular}{l}
\begin{tabular}{c} 
cubic \\
functions
\end{tabular}
\end{tabular} \\
\hline$A_{1 g}$ & +1 & +1 & +1 & +1 & +1 & +1 & +1 & +1 & +1 & +1 & - & $x^2+y^2, z^2$ & - \\
\hline$A_{2 g}$ & +1 & +1 & +1 & -1 & -1 & +1 & +1 & +1 & -1 & -1 & $\mathrm{R}_{\mathrm{Z}}$ & - & - \\
\hline $\mathrm{B}_{1 \mathrm{~g}}$ & +1 & -1 & +1 & +1 & -1 & +1 & -1 & +1 & +1 & -1 & - & $x^2-y^2$ & - \\
\hline $\mathrm{B}_{2 \mathrm{~g}}$ & +1 & 
-1 & +1 & -1 & +1 & +1 & -1 & +1 & -1 & +1 & - & $x y$ & - \\
\hline $\mathrm{E}_{\mathrm{g}}$ & +2 & 0 & -2 & 0 & 0 & +2 & 0 & -2 & 0 & 0 & $\left(\mathrm{R}_{\mathrm{x}}, \mathrm{R}_{\mathrm{y}}\right)$ & $(x z, y z)$ & - \\
\hline $\mathrm{A}_{1 \mathrm{u}}$ & +1 & +1 & +1 & +1 & +1 & -1 & -1 & -1 & -1 & -1 & - & - & - \\
\hline $\mathrm{A}_{2 \mathrm{u}}$ & +1 & +1 & +1 & -1 & -1 & 
-1 & -1 & -1 & +1 & +1 & $\mathrm{z}$ & - & $z^3, z\left(x^2+y^2\right)$ \\
\hline $\mathrm{B}_{1 \mathrm{u}}$ & +1 & -1 & +1 & +1 & -1 & -1 & +1 & -1 & -1 & +1 & - & - & $\mathrm{xyz}$ \\
\hline $\mathrm{B}_{2 \mathrm{u}}$ & +1 & -1 & +1 & -1 & +1 & 
-1 & +1 & -1 & +1 & -1 & - & - & $z\left(x^2-y^2\right)$ \\
\hline $\mathrm{E}_{\mathrm{u}}$ & +2 & 0 & -2 & 0 & 0 &
-2 & 0 & +2 & 0 & 0 & $(x, y)$ & - & $\left(\mathrm{xz}^2, \mathrm{yz}^2\right)\left(\mathrm{xy}^2, \mathrm{x}^2 \mathrm{y}\right),\left(\mathrm{x}^3, \mathrm{y}^3\right)$ \\
\hline
\end{tabular}%
}
\caption{Character table for point group $D_{4h}$ \citep{charactertableD4h}.}
\label{table:D4h}
\end{table}

Character tables for different groups can then be compared to find which irreps break symmetries between groups. For example, in Section \ref{exp:shape_def}, we consider deforming a cube with symmetry group $O_h$ to a rectangular prism with symmetry group $D_{4h}$. For $O_h$, the character table in Table \ref{table:Oh} shows that irrep $E_g$ has basis functions of $2z^2-x^2-y^2$ (highlighted in \textbf{\textcolor{blue}{blue}}) and $x^2-y^2$. For $D_{4h}$, the basis function of $2z^2-x^2-y^2$ is not present for any of the irreps. However, irrep $B_{1g}$ has a basis function of $x^2-y^2$. This implies that the basis function $2z^2-x^2-y^2$ for irrep $E_g$ can be used to break symmetry between $O_h$ and $D_{4h}$, but that the model should preserve $x^2 = y^2$.

\subsection{Converting Relaxed Weights to Spherical Signal}\label{app:sphharm}
For the shape deformation experiments in Section \ref{exp:shape_def}, we give the analytic forms of the spherical harmonics in $\texttt{e3nn}$. For the cube to rectangular prism, the learned relaxed weights for $\theta^{2,1}$ and $\theta^{4,1}$. For $\theta^{2,1}$, the non-zero relaxed weights are $\theta_0^{2,1}$ and $\theta_2^{2,1}$, corresponding to spherical harmonics
\begin{equation*}
    Y_0^{2,1} = \sqrt{5} \left (y^2 - \frac{1}{2} \left (x^2 + z^2 \right) \right), 
    Y_2^{2,1} = \frac{1}{2} \sqrt{15} (z^2 - x^2)
\end{equation*}
For a given model, we can then compute the scalar signal $f_0^2 = \theta_0^{2,1}Y_0^{2,1} + \theta_2^{2,1}Y_2^{2,1}$ for each layer, as seen in Table \ref{table:calcsig}. Note that the specific numerical values may change across layers/model initializations, yet they yield a scalar function of the form $ax^2+ay^2-2az^2$, demonstrating that the basis function for irrep $E_g$ is used to break symmetry while retaining $x^2 = y^2$ (note multiplying the basis function by -1 does not change the symmetry properties). A similar (but more involved) analysis can be done for relaxed weights $\theta^{4,1}$ to show that they also preserve $D_{4h}$ symmetry and transform as the irrep $E_g$.
\begin{table}[]
\centering
\begin{tabular}{|l|l|l|l|l|l|}
\hline
        & $\theta_0^{2,1}$ & $\theta_2^{2,1}$ & $x^2$ & $y^2$ & $z^2$  \\ \hline
Layer 0 & 0.0031           & -0.0054          & 0.007 & 0.007 & -0.014 \\ \hline
Layer 1 & 0.33             & -0.57            & 0.73  & 0.73  & -1.45  \\ \hline
\end{tabular}
\caption{Relaxed weights and calculated coefficients from the scalar signal for $x^2$, $y^2$, and $z^2$.}
\label{table:calcsig}
\end{table}
\section{Experimental Details}\label{app:exp}
We provide additional details for each experiment.
\subsection{Shape Deformations}
The objective is the MSE loss between the predicted spherical harmonic projections and the spherical harmonic projections of the vertices of the output shape. We use the SGD optimizer with a learning rate of 5e-3. The relaxed weights are also regularized throughout training by the $L_2$ norm with 
\begin{equation}\label{eq:regularize}
    \lambda \sum_{\substack{0\le l \le l_{\text{max}} \\ p \in \{1, -1\}}} ||\theta^{l, p}||_2
\end{equation}
where we set $\lambda = 1e-6$ in order to provide a slight bias towards equivariance. Each model is trained for 2,500 epochs.
\subsection{Particle in Electromagnetic Field}\label{app:exp_particle}
\begin{figure}[htb!]
	\centering
	\includegraphics[width=0.4\textwidth]{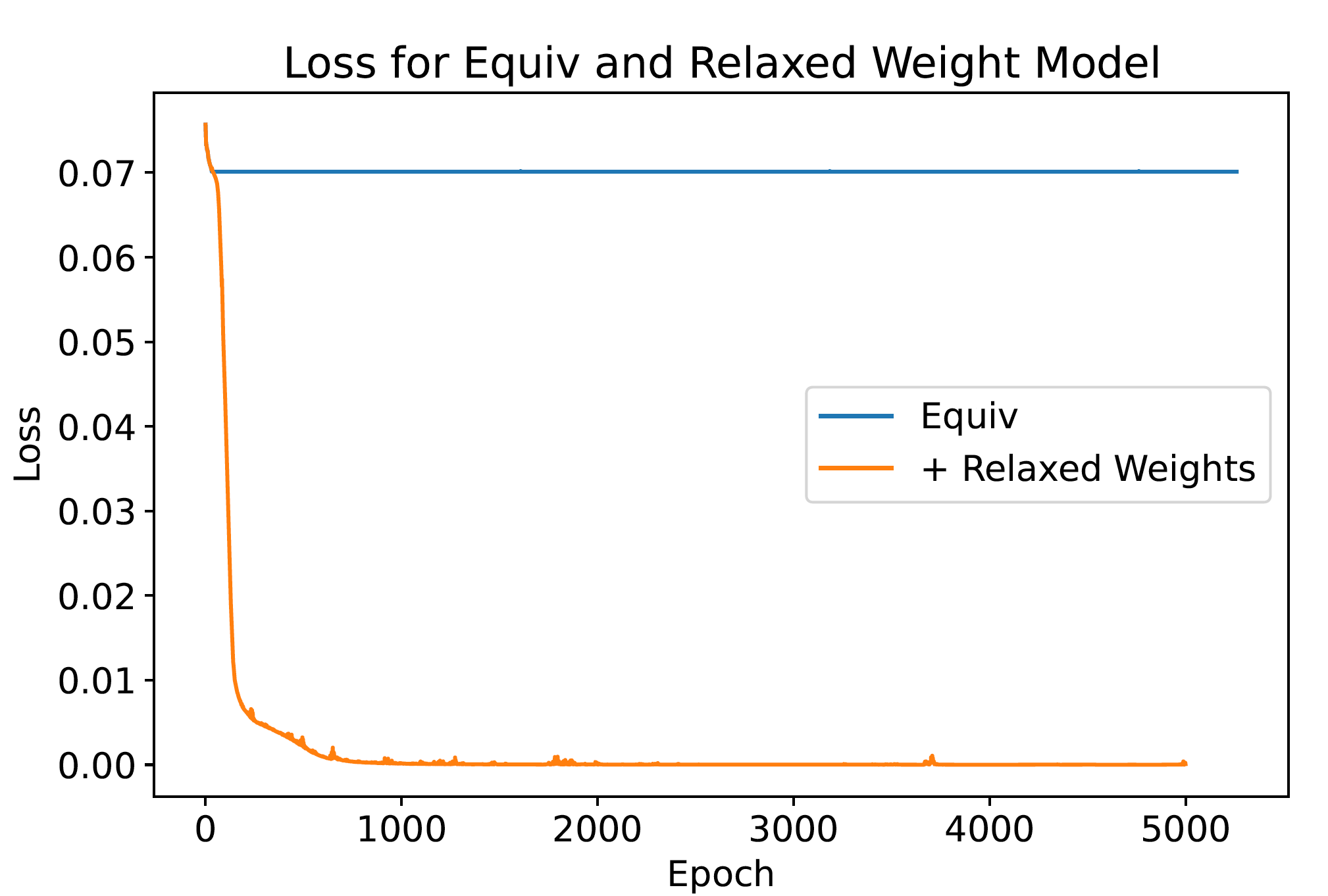}
	\caption{Loss for fully equivariant model and model with relaxed weights trained on the same task.}
	\label{fig:losses}
\end{figure}
We include no positional information to the network, so there are no edge attributes. The convolution over a single particle is still well-defined assuming a self-interaction edge is included in the network (reducing to an equivariant MLP where the irreps interact through the nonlinearities). The node inputs for each timestep are irreps $0_e \oplus 0_o \oplus 1_e \oplus 1_o \oplus 2_e \oplus 2_o$ with only the scalar term non-zero. Including higher order irreps allows for more paths in the tensor product and more model expressivity. 

We consider 500 timesteps from a single trajectory and train with a batch size of 50. We use the Adam optimizer with a learning rate of 1e-3. The model is trained for 5,000 epochs with the MSE loss between the predicted and true forces. The relaxed weights are also regularized by Equation \ref{eq:regularize} where we set $\lambda = 1e-4$. We show that an equivariant model cannot accomplish this task as it requires symmetry breaking. We plan to make our code publically available with further experiments.

\end{document}